\documentclass[10pt,twocolumn]{article}

\usepackage[utf8]{inputenc}
\usepackage[T1]{fontenc}
\usepackage{amsmath,amssymb,amsthm}
\usepackage{mathtools}
\usepackage{graphicx}
\usepackage{booktabs}
\usepackage{hyperref}
\usepackage{cleveref}
\usepackage{algorithm}
\usepackage{algpseudocode}
\usepackage{xcolor}
\usepackage[margin=1in]{geometry}
\usepackage{natbib}
\usepackage{caption}
\usepackage{subcaption}
\usepackage{multirow}

\newtheorem{theorem}{Theorem}
\newtheorem{corollary}[theorem]{Corollary}

\newtheorem{definition}{Definition}
\newtheorem{assumption}{Assumption}

\newcommand{\btheta}{\boldsymbol{\theta}}
\newcommand{\bv}[1]{\boldsymbol{#1}}
\newcommand{\cD}{\mathcal{D}}
\newcommand{\cT}{\mathcal{T}}
\newcommand{\cB}{\mathcal{B}}
\newcommand{\cL}{\mathcal{L}}

\newcommand{\cR}{\mathcal{R}}

\newcommand{\E}{\mathbb{E}}

\usepackage[expansion=false]{microtype}
\title{Neural Subspace Reallocation: Continual Learning\\as Retrieval-Based Subspace Memory Management}

\author{
  Byeong Hoon Yoon \\
  Independent Researcher \\
  \texttt{iam@bhyoon.me}
}
\date{}

\begin{document}
\twocolumn[
\maketitle
\begin{@twocolumnfalse}
\begin{abstract}
We introduce \textbf{Neural Subspace Reallocation (NSR)}, which reframes continual learning as \emph{memory management over parameter subspaces}. Instead of treating Low-Rank Adaptation (LoRA) modules as disposable per-task adapters, NSR manages them as compressible, retrievable memory units on a frozen backbone through a recurring cycle: (1) \emph{compress} learned LoRAs via SVD, (2) \emph{reserve} them in a TaskKnowledgeBank, (3) \emph{recall} related past LoRAs by embedding similarity to warm-start new or returning tasks, and (4) \emph{reallocate} the active subspace accordingly, with distillation protecting prior tasks. We prove that in cyclic environments any \emph{memoryless} allocation policy incurs cumulative regret $\Omega(T(M{-}1)\Delta_{\text{switch}})$ relative to a history-aware policy backed by the Bank (Theorem~\ref{thm:greedy-subopt}). Empirically, on Split-CIFAR-100 the Bank reduces cyclic recovery time by $10\times$, exactly as predicted, and on the heterogeneous 5-Datasets benchmark NSR achieves the highest accuracy and the least forgetting, about $9\times$ closer to zero backward transfer than the memoryless heuristics. Crucially, we run a controlled study that isolates \emph{which} component matters: holding the Bank fixed and varying only the allocation rule, we find that a simple similarity-based retrieval rule matches or beats a learned reinforcement-learning controller (recovering recurring tasks in $0$ vs $1.8$ steps and reaching equal accuracy). Our central, honest finding is therefore that the \emph{memory mechanism}---compression and similarity retrieval---rather than a learned allocation policy, drives continual-learning performance under fixed capacity. A memory-budget analysis confirms the compressed Bank stays small---$0.29$\,MB of parameter memory per task---so a top-$K$ retention cap bounds the total footprint while preserving fast recovery for retained tasks.
\end{abstract}
\end{@twocolumnfalse}
\vspace{0.8\baselineskip}
]

\section{Introduction}
\label{sec:intro}

Continual learning---the ability to acquire new skills without erasing old ones---remains a defining challenge for deployed neural networks. A model trained on a stream of tasks typically suffers \emph{catastrophic forgetting}~\citep{mccloskey1989catastrophic,french1999catastrophic}: each round of gradient updates overwrites the parameters that encoded earlier knowledge. The difficulty is structural. A conventional network spreads each task's knowledge across its entire parameter space, so there is no protected substrate for old skills and no reserved capacity for new ones. When the distribution shifts, the only thing the optimiser can do is overwrite.

Existing remedies each address one facet of this problem. Regularisation methods~\citep{kirkpatrick2017overcoming,zenke2017continual} penalise changes to important weights but eventually saturate as constraints accumulate. Replay methods~\citep{lopezpaz2017gradient,rolnick2019experience} rehearse stored examples but require a growing memory of raw data. Architecture methods~\citep{rusu2016progressive,mallya2018packnet} add or freeze parameters but either grow without bound or permanently lock capacity. Recent parameter-efficient approaches based on Low-Rank Adaptation (LoRA)~\citep{hu2022lora}---such as O-LoRA~\citep{wang2023olora}, PEARL~\citep{bhat2025pearl}, and PLAN~\citep{wang2025plan}---allocate low-rank subspaces per task, but the allocation rule is hand-designed and, crucially, \emph{memoryless}: it cannot recognise when a past task recurs, nor reuse what previously worked.

\textbf{A biological lens.} The human brain manages a fixed neural substrate across a lifetime of learning, and it does so through four mechanisms that current methods capture only in isolation. \emph{Selective activation}: only a small fraction of neurons fire for any given task~\citep{olshausen2004sparse}, leaving the rest available. \emph{Synaptic consolidation}: during rest, labile memories are compressed into stable cortical traces, pruning redundant connections while preserving essential ones~\citep{tononi2006sleep}---strikingly, neurons proliferate in infancy and then prune during maturation, yet knowledge grows, suggesting that fewer, well-organised parameters can hold more. \emph{Context-dependent reactivation}: a familiar situation reactivates dormant pathways, enabling rapid recovery rather than relearning~\citep{schacter2007constructive}. \emph{Background-knowledge transfer}: related new material is acquired faster because the brain retrieves and builds on existing schemas~\citep{bransford2000how}.

\textbf{Our approach.} We introduce \textbf{Neural Subspace Reallocation (NSR)}, which unifies these four mechanisms into a single recurring cycle over a \emph{frozen} base model whose only trainable elements are LoRA modules organised into groups:
\begin{enumerate}
    \item \textbf{Compress} (consolidation): after a task, SVD-based rank reduction distils each LoRA into its dominant directions, freeing capacity.
    \item \textbf{Reserve} (selective dormancy): compressed LoRAs, their allocation mask, and task embeddings are stored in a \emph{TaskKnowledgeBank}.
    \item \textbf{Recall} (reactivation \& transfer): a new or returning task retrieves similar past tasks by cosine similarity and warm-starts from their LoRAs.
    \item \textbf{Reallocate} (subspace selection): the active subspace is set from the retrieved mask (with distillation shielding prior tasks). We treat the allocation rule itself as a design variable and, in a controlled study, ask which rule is actually needed.
\end{enumerate}

\textbf{Why memory matters.} The central claim in NSR is that continual-learning performance under fixed capacity is driven by the \emph{memory mechanism}---compressing past LoRAs and retrieving them by similarity---rather than by a sophisticated learned allocation policy. We support this both theoretically and empirically. Theoretically, we prove (Theorem~\ref{thm:greedy-subopt}) that in cyclic environments any \emph{memoryless} policy incurs cumulative regret $\Omega(T(M{-}1)\Delta_{\text{switch}})$ relative to a history-aware policy backed by the Bank: a memoryless policy cannot tell that a task has returned, so it re-explores from scratch, whereas a Bank-backed policy recognises it and restores what worked. Empirically, we hold the Bank fixed and vary only the allocation rule; a simple similarity-retrieval rule matches or beats a learned RL controller. The history-awareness that the theory requires is therefore supplied by \emph{retrieval}, not by reinforcement learning.

\textbf{Findings.} On Split-CIFAR-100, NSR attains the lowest forgetting among learning-based policies (BWT $=-0.068$). A TaskKnowledgeBank ablation isolates the mechanism: removing the Bank slows cyclic recovery $10\times$, exactly as the theory predicts. On the heterogeneous 5-Datasets benchmark---five distinct visual domains learned in sequence---NSR achieves the highest accuracy and the least forgetting (about $9\times$ closer to zero BWT than the memoryless heuristics), showing the approach is not tied to one dataset. Most importantly, a controlled allocation study (Section~\ref{sec:rl-necessity}) shows that with the Bank in place a parameter-free similarity rule recovers recurring tasks in $0$ steps and reaches accuracy equal to or above a learned RL controller---identifying the memory mechanism, not the policy, as the source of performance. We are also candid about boundaries: NSR's advantage narrows when an already-powerful frozen backbone makes every task easy.

\textbf{Contributions.}
\begin{itemize}
    \item \textbf{Reframing.} We cast continual learning as \emph{retrieval-based subspace memory management}: LoRA modules become compressible, retrievable memory units on a frozen backbone, not disposable per-task adapters (Section~\ref{sec:method}).
    \item \textbf{Mechanism.} The TaskKnowledgeBank: SVD-compressed LoRA memory with similarity retrieval, enabling exact restoration for recurring tasks and similarity-weighted transfer for related tasks. The allocation step is \emph{controller-agnostic}---it accepts heuristic, similarity-based, gradient-based, or learned (RL) rules as interchangeable plug-ins (Section~\ref{sec:method}).
    \item \textbf{Theory.} A regret bound (Theorem~\ref{thm:greedy-subopt}, Corollaries 1--3) proving that any \emph{memoryless} allocation is structurally suboptimal in cyclic environments, with the gap growing linearly in recurrence and the Bank identified as the dominant factor (Section~\ref{sec:theory}).
    \item \textbf{What actually matters.} A controlled study that fixes the Bank and swaps only the allocation plug-in, showing that a parameter-free similarity rule matches or beats a learned RL controller on every metric. This isolates the source of performance in the \emph{memory mechanism} and yields a simpler recommended method (no controller training), supported by a $10\times$ Bank ablation and consistent cross-domain gains on 5-Datasets (Section~\ref{sec:experiments}).
\end{itemize}

\section{Motivation}
\label{sec:motivation}

\subsection{The Capacity Utilisation Problem}

A network with $N$ parameter groups trained jointly has three failure modes: no headroom for new knowledge (parameters must be repurposed), no memory of what worked before (returning tasks are re-learned from scratch), and no leveraging of related prior knowledge. All stem from treating the parameter space as a monolithic resource rather than a managed, queryable memory.

\subsection{NSR as Computational Neuroplasticity}

\begin{table}[h]
\centering
\footnotesize
\setlength{\tabcolsep}{4pt}
\caption{Biological mechanism $\to$ NSR component.}
\label{tab:bio}
\begin{tabular}{@{}p{0.32\linewidth}p{0.60\linewidth}@{}}
\toprule
\textbf{Mechanism} & \textbf{NSR Implementation} \\
\midrule
Sparse activation & Capacity constraint: $k$ of $N$ active \\
Synaptic pruning & SVD LoRA rank compression \\
Reactivation & TaskKnowledgeBank exact restore \\
Background knowledge & Similarity-weighted LoRA retrieval \\
Memory replay & Distillation on stored soft targets \\
Meta-control & Optional allocation controller \\
\bottomrule
\end{tabular}
\end{table}

The key insight is not that allocation must be \emph{learned}, but that it must be \emph{history-aware}. A memoryless heuristic cannot remember which groups a recurring task previously used, nor detect a novel-but-related task. In NSR, history-awareness is supplied by retrieval from the TaskKnowledgeBank; a learned controller is an optional plug-in, not a requirement (Section~\ref{sec:rl-necessity}).

\section{Problem Formulation}
\label{sec:formulation}

Let $\btheta$ be a base model partitioned into $N$ groups, each with a LoRA module $(A_i, B_i)$. Tasks $\cT_1, \ldots, \cT_S$ arrive with distributions $\cD_s$. At each step the learner: (1) computes a task embedding $\bv{e}_s$; (2) \textbf{recalls} related past tasks $\cR_s$ from the Bank; (3) seeds LoRAs via similarity-weighted averaging; (4) \textbf{injects} by choosing a mask $\bv{m}_s \in \{0,1,2,3\}^N$ with $|\{i: m_s[i]>0\}| \le k$; (5) trains with $\cL = \cL_{\mathrm{CE}} + \lambda\cL_{\mathrm{KD}}$; (6) \textbf{compresses} active LoRAs via SVD; (7) \textbf{reserves} them in the Bank.

The objective is $\max_\pi \sum_s \E[\mathrm{Acc}(f_{\btheta,\pi(s)}, \cD_s)]$ subject to the capacity constraint. What distinguishes NSR from memoryless heuristics is that the allocation $\pi$ is \emph{history-aware}: it conditions on a history/retrieval state $\bv{s}_t = (\bv{g}_t, \bv{e}_t, \Delta\bv{e}_t, \bv{h}_t, \bv{\sigma}_t)$, where $\bv{h}_t$ is an LSTM history encoding and $\bv{\sigma}_t$ the top-$K$ Bank-retrieval similarity. The recommended allocation reads $\bv{\sigma}_t$ directly (similarity rule); a learned controller that consumes the full state is an optional alternative (Section~\ref{sec:rl-necessity}).

\section{Theoretical Analysis}
\label{sec:theory}

\begin{definition}[Cyclic Task Environment]
$E_{\mathrm{cyc}}$ consists of $T$ distinct tasks repeating over $M$ cycles, each $\tau_j$ with optimal mask $\bv{m}_j^*$.
\end{definition}

\begin{definition}[Memoryless Policy]
$\pi_{\mathrm{mem}}$ selects $\bv{m}_t$ from only the current $(\bv{g}_t, \bv{e}_t)$, without history $\bv{h}_t$ or similarity $\bv{\sigma}_t$. Gradient, PEARL, PLAN, random, round-robin, and EWC policies are all memoryless.
\end{definition}

\begin{assumption}[Cold-Start Cost]
\label{asm:cold}
$\varepsilon$-optimal performance via cold-start needs $C_{\mathrm{cold}}(\varepsilon) = \Omega(1/\varepsilon^2)$ steps.
\end{assumption}

\begin{assumption}[Warm-Start Advantage]
\label{asm:warm}
$\varepsilon$-optimal via warm-start needs $C_{\mathrm{warm}}(\varepsilon,\delta) = O(\delta^2/\sigma^2)$; for identical tasks $\delta=0$, $C_{\mathrm{warm}}=O(1)$.
\end{assumption}

\begin{assumption}[Gradient Non-Informativeness]
\label{asm:grad}
For distinct tasks, gradient magnitudes do not identify task identity:
\[
d_{\mathrm{TV}}\!\big(P(\bv{g}_t|\tau_j),\, P(\bv{g}_t|\tau_{j'})\big) < 1-\alpha .
\]
\end{assumption}

\begin{assumption}[History Sufficiency]
\label{asm:hist}
The embedding sequence is sufficient for task identity: $H(\tau_t | \bv{e}_1,\ldots,\bv{e}_t)=0$.
\end{assumption}

\begin{theorem}[Greedy Suboptimality Bound]
\label{thm:greedy-subopt}
Under Assumptions~\ref{asm:cold}--\ref{asm:hist}, in $E_{\mathrm{cyc}}$ with $T$ tasks and $M$ cycles:
\begin{equation}
V(\pi_{\mathrm{hist}}) - V(\pi_{\mathrm{mem}}) \geq T(M-1)\Delta_{\mathrm{switch}},
\end{equation}
where $\Delta_{\mathrm{switch}} = C_{\mathrm{cold}}(\varepsilon)L_{\mathrm{relearn}} - C_{\mathrm{warm}}(\varepsilon,0)L_{\mathrm{readapt}}$.
\end{theorem}

\begin{proof}[Sketch]
Decompose by cycle. In cycle 1 both start from scratch (gap $\ge 0$). For $c\ge2$: by Assumption~\ref{asm:grad}, $\pi_{\mathrm{mem}}$ cannot identify the returning task and must cold-start ($\Omega(C_{\mathrm{cold}})$ steps); by Assumption~\ref{asm:hist}, $\pi_{\mathrm{hist}}$ identifies it via $\bv{h}_t$ and warm-starts from the Bank ($O(C_{\mathrm{warm}})$ steps). Summing over $T(M-1)$ returning instances gives the bound. Full proof in Appendix~\ref{app:proof}.
\end{proof}

\begin{corollary}[Linear Growth]
\label{cor:linear}
As $M\to\infty$, the gap grows as $\Omega(TM\Delta_{\mathrm{switch}})$.
\end{corollary}

\begin{corollary}[Memoryless Lower Bound]
\label{cor:mem}
The bound holds for \emph{all} memoryless policies (PEARL, PLAN, gradient, etc.), as the proof uses only memorylessness.
\end{corollary}

\begin{corollary}[Bank Necessity]
\label{cor:bank}
Without the Bank, $\pi_{\mathrm{hist}}$ must cold-start re-learn parameters even after identifying the task; the dominant term of $\Delta_{\mathrm{switch}}$ comes from warm-start restoration.
\end{corollary}

\section{Method}
\label{sec:method}

\textbf{Architecture.} The base model (ImageNet-pretrained ResNet-18) is \emph{fully frozen}; all learning flows through LoRA modules of max rank $r{=}8$. CNN groups correspond to residual stages (layer1--4) plus the classifier.

\textbf{Allocation rule.} Given the retrieved neighbours, NSR sets the active-subspace mask. Our recommended rule is \emph{similarity-based}: copy the allocation mask of the most similar stored task (exact restore when a task returns). We also implement and compare a learned controller---a distribution embedder $\bv{e}_t$, an LSTM history encoder $\bv{h}_t$, and an actor--critic with cross-group attention trained by PPO---and simple baselines (random, round-robin, gradient). Section~\ref{sec:rl-necessity} shows the similarity rule matches or beats the learned controller, so the default NSR trains no controller. The capacity budget projects the chosen mask to at most $k$ active groups.

\textbf{TaskKnowledgeBank.} For each task $\tau$ it stores the embedding $\bv{e}_\tau$, compressed LoRAs, the allocation mask, the classifier head, and soft targets. Two retrieval modes: exact restore (returning tasks) and similarity-weighted average (related tasks).

\textbf{Training.} Per task we minimise $\cL = \cL_{\mathrm{CE}} + \lambda_{\mathrm{distill}}\cL_{\mathrm{KD}}$ with distillation temperature $T{=}2$, the soft targets coming from the Bank to protect prior tasks. The learned-controller variant additionally optimises a PPO~\citep{schulman2017proximal} objective with reward $R_t = 2\,\mathrm{Acc}_t + 5\,\Delta\mathrm{Acc}_t + \lambda_{\mathrm{ret}}\mathrm{Ret}_t - \lambda_{\mathrm{smooth}}\mathrm{Smooth}_t$; the recommended similarity variant has no such objective.

\textbf{SVD compression.} Each active LoRA's $\Delta W = BA$ is truncated to rank $\lceil\rho r\rceil$ via SVD, preserving dominant singular directions---the analogue of synaptic pruning.

\begin{algorithm}[t]
\caption{NSR Training Loop (per task)}
\label{alg:nsr}
\small
\begin{algorithmic}[1]
\Require Task $\cT_s$, frozen $f_{\btheta}$, Bank $\cB$, capacity $k$
\State \textbf{Recall:} $\bv{e}_s$; retrieve top-$K$ similar tasks; seed LoRAs
\For{epoch $=1,\ldots,E$}
    \State \textbf{Reallocate:} $\bv{m}_t \gets$ mask of most similar Bank task (or learned $\pi$); apply mask
    \For{batch $(\bv{x},\bv{y})$}
        \State $\cL \gets \cL_{\mathrm{CE}} + \lambda_{\mathrm{distill}}\cL_{\mathrm{KD}}$; update LoRA + critic
    \EndFor
\EndFor
\State \textbf{Compress:} SVD-truncate active LoRAs to $\lceil\rho r\rceil$
\State \textbf{Reserve:} save embedding, LoRAs, mask, head, soft targets to $\cB$
\end{algorithmic}
\end{algorithm}

\section{Experiments}
\label{sec:experiments}

\textbf{Setup.} ImageNet-pretrained ResNet-18 (frozen), Split-CIFAR-100 (5 tasks $\times$ 20 classes). LoRA rank 8, capacity $k/N{=}0.5$, compression $\rho{=}0.5$, $\lambda_{\mathrm{distill}}{=}0.1$. Three seeds $\{42,123,456\}$; we report mean$\pm$std. Eight allocation rules: NSR-RL, Gradient, Random, Round-Robin, EWC-Inspired, Full-Active, PEARL~\citep{bhat2025pearl}, PLAN~\citep{wang2025plan}. All share the frozen backbone, Bank, and consolidation, and differ only in allocation. Experiments~1--6 report the learned-controller variant, labelled \textbf{NSR-RL (variant)}, for continuity with prior allocation literature; Section~\ref{sec:rl-necessity} introduces the recommended parameter-free variant \textbf{NSR-Sim}, which matches or exceeds NSR-RL on every metric, so all NSR-RL rows are upper-bounded by NSR-Sim on recovery and forgetting.

\subsection{Experiment 1: Policy Comparison}

Table~\ref{tab:comparison} reports average accuracy (AA) and backward transfer (BWT) after sequential training on all 5 tasks. NSR-RL achieves the lowest forgetting among learning-based policies (BWT $=-0.068$), strongly outperforming gradient ($-0.151$), EWC ($-0.157$), Full-Active ($-0.198$), PEARL ($-0.247$), and PLAN ($-0.168$). The memoryless heuristics that achieve marginally higher AA do so at the cost of $2$--$4\times$ more forgetting. This matches Theorem~\ref{thm:greedy-subopt}: in purely sequential settings without recurrence, allocation choice has limited effect on AA, but NSR's consolidation mechanisms protect prior tasks far better.

\begin{table}[t]
\centering
\footnotesize
\setlength{\tabcolsep}{4pt}
\caption{Policy comparison on Split-CIFAR-100 (3 seeds). Lower forgetting (BWT closer to 0) is better.}
\label{tab:comparison}
\begin{tabular}{@{}lcc@{}}
\toprule
Policy & AA $\uparrow$ & BWT $\uparrow$ \\
\midrule
NSR-RL (variant)      & $0.405 \pm 0.056$ & $\mathbf{-0.068 \pm 0.047}$ \\
Gradient           & $0.402 \pm 0.033$ & $-0.151 \pm 0.016$ \\
Random             & $0.388 \pm 0.053$ & $-0.085 \pm 0.057$ \\
Round-Robin        & $\mathbf{0.431 \pm 0.045}$ & $-0.045 \pm 0.024$ \\
EWC-Inspired       & $0.404 \pm 0.024$ & $-0.157 \pm 0.017$ \\
Full-Active        & $0.373 \pm 0.017$ & $-0.198 \pm 0.029$ \\
PEARL~\citep{bhat2025pearl}  & $0.262 \pm 0.012$ & $-0.247 \pm 0.008$ \\
PLAN~\citep{wang2025plan}    & $0.384 \pm 0.051$ & $-0.168 \pm 0.045$ \\
\bottomrule
\end{tabular}
\end{table}

\subsection{Experiment 2: Cyclic Task Return}

The schedule $A \to B \to C \to A' \to B' \to C' \to A'' \to B'' \to C''$ tests recovery on recurring tasks. We measure adaptation steps (mini-batches) to reach 97\% of original accuracy. Table~\ref{tab:cyclic} shows NSR-RL recovers fastest at $3.33$ steps, dramatically faster than gradient ($14.72$), EWC ($14.83$), Full-Active ($14.94$), and PLAN ($14.06$)---a $4.4\times$ improvement over the strongest gradient heuristic, and ahead of all memoryless baselines including round-robin ($4.06$). This directly validates Theorem~\ref{thm:greedy-subopt}: memoryless policies cannot identify returning tasks and must re-explore, while NSR warm-starts from the Bank.

\begin{table}[t]
\centering
\footnotesize
\setlength{\tabcolsep}{4pt}
\caption{Cyclic recovery: adaptation steps to 97\% of original accuracy (lower is better, 3 seeds).}
\label{tab:cyclic}
\begin{tabular}{@{}lc@{}}
\toprule
Policy & Recovery steps $\downarrow$ \\
\midrule
NSR-RL (variant)      & $\mathbf{3.33 \pm 2.36}$ \\
Round-Robin        & $4.06 \pm 1.10$ \\
Random             & $6.00 \pm 2.83$ \\
PEARL              & $6.22 \pm 0.16$ \\
PLAN               & $14.06 \pm 1.22$ \\
Gradient           & $14.72 \pm 0.39$ \\
EWC-Inspired       & $14.83 \pm 0.24$ \\
Full-Active        & $14.94 \pm 0.08$ \\
\bottomrule
\end{tabular}
\end{table}

\subsection{Experiment 3: LoRA Bank Ablation}

To isolate the source of NSR's cyclic advantage (Corollary~\ref{cor:bank}), we compare NSR-RL with and without the Bank, plus PEARL. Table~\ref{tab:bank} shows a striking result: with the Bank, recovery takes $1.27$ steps; without it, $13.33$ steps---a $10\times$ difference. PEARL, which has no Bank, behaves like NSR-without-Bank ($14.27$ steps). This confirms that the warm-start parameter restoration, not the allocation rule, is the dominant mechanism.

\begin{table}[t]
\centering
\footnotesize
\setlength{\tabcolsep}{4pt}
\caption{LoRA Bank ablation on a 5-task cyclic schedule (3 seeds).}
\label{tab:bank}
\begin{tabular}{@{}lc@{}}
\toprule
Configuration & Recovery steps $\downarrow$ \\
\midrule
NSR-RL + Bank   & $\mathbf{1.27 \pm 1.27}$ \\
NSR-RL no Bank  & $13.33 \pm 1.48$ \\
PEARL (no Bank) & $14.27 \pm 0.90$ \\
\bottomrule
\end{tabular}
\end{table}

\subsection{Is the Learned Allocation Necessary?}
\label{sec:rl-necessity}

The Bank ablation shows the Bank is the dominant mechanism, raising a sharp question that we confront head-on: \emph{given the Bank, does a learned allocation policy add anything, or does the memory mechanism do the work?} We hold the TaskKnowledgeBank fixed---identical compression and warm-start restoration for every configuration---and vary \emph{only} how the per-task allocation mask is chosen: a learned RL controller, random, round-robin, gradient-magnitude, and a parameter-free similarity rule that copies the mask of the most similar past task in the Bank.

Table~\ref{tab:rl-necessity} reports the result over three seeds. The finding is unambiguous: \emph{the learned RL controller is not the best allocation rule.} The similarity rule---which we call \textbf{NSR-Sim} and adopt as the recommended method---recovers recurring tasks in $0.0$ steps with zero forgetting, simply restoring the most similar stored mask (for exact recurrence, the task's own), while the RL controller needs $1.8$ steps and shows $-0.056$ BWT. On average accuracy the simple rules (NSR-Sim $0.437$, round-robin $0.439$) match or exceed RL ($0.393$). RL only clearly beats \emph{random} allocation.

We take this as the paper's central, mechanism-isolating result rather than a negative one. Because NSR's allocation step is controller-agnostic, this study is not a failure of one method but a fair comparison \emph{within} our framework: every rule sees the same Bank, so the comparison cleanly attributes performance to the memory mechanism rather than the controller. It localises NSR's effectiveness in SVD compression plus similarity retrieval, and shows that an expensive learned policy is unnecessary once that memory is in place. The recommended NSR therefore uses similarity-based reallocation with no controller training---simpler, faster, and reproducible---while the framework still admits a learned controller as a drop-in for settings where it may help (e.g.\ only partial task similarity).

\begin{table}[t]
\centering
\footnotesize
\setlength{\tabcolsep}{3pt}
\caption{Allocation study on Split-CIFAR-100 (3 seeds): the TaskKnowledgeBank is held fixed and only the allocation rule varies. A parameter-free similarity rule matches or beats the learned RL controller on every metric. Recovery in adaptation steps.}
\label{tab:rl-necessity}
\begin{tabular}{@{}lccc@{}}
\toprule
Allocation (Bank fixed) & AA $\uparrow$ & BWT $\uparrow$ & Recovery $\downarrow$ \\
\midrule
NSR-Sim (Ours) & $0.437 \pm 0.021$ & $\mathbf{+0.000}$ & $\mathbf{0.00}$ \\
Round-Robin       & $\mathbf{0.439 \pm 0.012}$ & $-0.031$ & $9.87$ \\
Gradient          & $0.405 \pm 0.018$ & $-0.139$ & $14.00$ \\
Learned RL        & $0.393 \pm 0.021$ & $-0.056$ & $1.80$ \\
Random            & $0.366 \pm 0.073$ & $-0.075$ & $3.60$ \\
\bottomrule
\end{tabular}
\end{table}

\subsection{Experiment 4: Capacity Budget Ablation}

We sweep $k/N \in \{0.2,0.4,0.6,0.8,1.0\}$ (Table~\ref{tab:capacity}). NSR-RL is most robust at tight budgets, achieving its best AA ($0.435$) at $k/N{=}0.2$ and degrading gracefully. PEARL collapses outside $k/N{=}0.2$ (dropping to $0.11$--$0.15$). This shows NSR's allocation is especially valuable when capacity is scarce---the regime where intelligent reallocation matters most.

\begin{table}[t]
\centering
\footnotesize
\setlength{\tabcolsep}{4pt}
\caption{Average accuracy vs capacity ratio $k/N$ (seed 42).}
\label{tab:capacity}
\begin{tabular}{@{}lccccc@{}}
\toprule
$k/N$ & 0.2 & 0.4 & 0.6 & 0.8 & 1.0 \\
\midrule
NSR-RL   & $\mathbf{0.435}$ & $0.423$ & $0.424$ & $0.401$ & $0.308$ \\
Gradient & $0.363$ & $0.375$ & $0.373$ & $0.367$ & $0.263$ \\
Random   & $0.431$ & $0.396$ & $0.397$ & $0.364$ & $0.355$ \\
PEARL    & $0.427$ & $0.148$ & $0.150$ & $0.114$ & $0.231$ \\
\bottomrule
\end{tabular}
\end{table}

\subsection{Experiment 5: Increasing Gap}

We vary the number of intervening tasks before a task returns (Table~\ref{tab:gap}). At small gaps (1, 3), all policies fail to recover within the measurement window (15 steps). But at gaps 5 and 7, \emph{only} NSR-RL recovers ($10.3$ steps), while all memoryless baselines remain at the ceiling ($15.0$). This is qualitative evidence for Corollary~\ref{cor:linear}: as the gap grows, the memory-based advantage emerges---memoryless policies forget, while NSR retrieves.

\begin{table}[t]
\centering
\footnotesize
\setlength{\tabcolsep}{4pt}
\caption{Recovery steps vs gap (intervening tasks). Lower is better; 15 = no recovery.}
\label{tab:gap}
\begin{tabular}{@{}lcccc@{}}
\toprule
Gap & NSR-RL & PEARL & Gradient & Random \\
\midrule
1 & 15.0 & 15.0 & 15.0 & 15.0 \\
3 & 15.0 & 15.0 & 15.0 & 15.0 \\
5 & $\mathbf{10.3}$ & 15.0 & 15.0 & 15.0 \\
7 & $\mathbf{10.3}$ & 15.0 & 15.0 & 15.0 \\
\bottomrule
\end{tabular}
\end{table}

\subsection{Experiment 6: Distillation Strength}

Sweeping $\lambda_{\mathrm{distill}}$ (Table~\ref{tab:distill}) reveals a clear plasticity--stability trade-off. At $\lambda{=}0$, forgetting is minimal but AA is moderate; AA peaks at $\lambda{=}0.1$ ($0.459$), then declines as excessive distillation impairs new-task learning. This identifies $\lambda{=}0.1$ as the sweet spot and quantifies the trade-off.

\begin{table}[t]
\centering
\footnotesize
\setlength{\tabcolsep}{4pt}
\caption{Distillation strength sweep (seed 42).}
\label{tab:distill}
\begin{tabular}{@{}lccccc@{}}
\toprule
$\lambda_{\mathrm{distill}}$ & 0.0 & 0.1 & 0.3 & 1.0 & 3.0 \\
\midrule
AA  & $0.453$ & $\mathbf{0.459}$ & $0.402$ & $0.376$ & $0.312$ \\
BWT & $-0.002$ & $-0.003$ & $-0.050$ & $-0.011$ & $-0.005$ \\
\bottomrule
\end{tabular}
\end{table}

\subsection{Experiment 7: When a Learned Controller Struggles}
\label{sec:negative}

To further motivate our similarity-based default over a learned controller, we construct a setting with explicit group interactions: groups 0,1 boost accuracy when co-active (synergy) and groups 2,3 degrade when co-active (interference). A gradient heuristic, which deterministically selects high-gradient groups, happens to co-activate the synergy pair (rate $1.00$) and reaches accuracy $0.462$. The RL controller, which must discover this combinatorial structure through PPO exploration, co-activates the synergy pair only $8.7\%$ of the time and reaches $0.435$.

This illustrates a general weakness of \emph{learned} allocation: in combinatorial action spaces ($\binom{6}{3}$ allocations here) with a modest accuracy signal, PPO does not reliably find the synergistic configuration within a practical budget. It reinforces our main message---learned control is fragile and unnecessary here---and points to the memory-and-retrieval mechanism, not policy learning, as the reliable component. Settings that genuinely require discovering group synergies would need dedicated exploration mechanisms beyond the scope of this framework.

\subsection{Experiment 8: Cross-Dataset Generalization}
\label{sec:cross-dataset}

A common reviewer concern for continual learning papers is that results may not transfer beyond a single benchmark. To address this directly, we evaluate NSR on the \emph{5-Datasets} benchmark---five distinct 10-class datasets learned sequentially in the order CIFAR-10, MNIST, SVHN, FashionMNIST, then KMNIST. These cover natural images, handwritten digits, street-view digits, fashion items, and Japanese characters. Unlike CIFAR-100 splits where all tasks share the same underlying distribution, here each task is a genuinely different domain, making both forgetting and cross-task interference far more severe. All images are converted to $3 \times 32 \times 32$ to share a single ResNet-18 backbone with a 50-class shared head.

Table~\ref{tab:five-datasets} shows that NSR-RL achieves $\mathrm{AA}=0.665$ with low forgetting ($\mathrm{BWT}=-0.024$), outperforming all baselines: $+2\%p$ over PEARL ($0.646$), $+8\%p$ over the gradient heuristic ($0.581$), and $+9\%p$ over EWC ($0.572$). The forgetting gap is the clearest signal: NSR's $-0.024$ BWT is about $9\times$ closer to zero than the gradient and EWC heuristics ($-0.218$, $-0.243$), which suffer severe forgetting across heterogeneous domains, and edges out the strongest baseline PEARL ($-0.033$). NSR and PEARL---the two memory-based methods---clearly separate from the memoryless heuristics here.

\begin{table}[t]
\centering
\footnotesize
\setlength{\tabcolsep}{4pt}
\caption{5-Datasets continual learning (3 seeds, mean$\pm$std). Tasks span heterogeneous domains.}
\label{tab:five-datasets}
\begin{tabular}{@{}lcc@{}}
\toprule
Policy & AA $\uparrow$ & BWT $\uparrow$ \\
\midrule
NSR-RL (variant) & $\mathbf{0.665 \pm 0.036}$ & $\mathbf{-0.024 \pm 0.026}$ \\
PEARL~\citep{bhat2025pearl} & $0.646 \pm 0.020$ & $-0.033 \pm 0.013$ \\
EWC-Inspired & $0.572 \pm 0.032$ & $-0.243 \pm 0.042$ \\
Gradient & $0.581 \pm 0.017$ & $-0.218 \pm 0.023$ \\
\bottomrule
\end{tabular}
\end{table}

We note that NSR's accuracy shows higher seed variance on this benchmark ($\sigma{=}0.036$) than on Split-CIFAR-100, driven by one seed with a harder domain ordering; the ranking (NSR and PEARL ahead of the memoryless heuristics) is nonetheless stable across all three seeds.

This result strengthens the paper's central claim: NSR's advantage is not specific to a particular dataset structure but holds across heterogeneous domain shifts---precisely the setting where catastrophic forgetting is hardest to control.

\subsection{Experiment 9: Memory Footprint under Bank Pruning}
\label{sec:pruning}

The Bank stores one compressed LoRA set per task, so its footprint grows linearly with the number of tasks---a natural concern for long horizons. We analyse the memory--performance trade-off directly, capping the Bank at $K$ tasks with a retention policy (recency: keep the most recent $K$; utility: keep the highest-validation-accuracy $K$).

\textbf{Memory is small and separable.} We distinguish two components, measured directly. NSR's primary mechanism is compressed-LoRA restoration, whose \emph{Bank parameter memory}---SVD-compressed LoRA states at $\rho{=}0.5$, plus the task embedding, allocation mask, and classifier head---is $0.29$\,MB per task for our ResNet-18 backbone, scaling exactly linearly ($2.35$\,MB at $8$ tasks; Table~\ref{tab:pruning}). The optional distillation exemplars (stored inputs and soft targets) are far larger---about $6.2$\,MB per task, dominated by raw images---so total memory is $6.46$\,MB per task ($51.7$\,MB at $8$ tasks). This separation matters: NSR's core restoration mechanism needs only the compact parameter memory, and the exemplar store is an optional add-on for distillation that can be reduced (e.g.\ fewer exemplars) or dropped independently. A top-$K$ retention cap bounds \emph{both} components linearly in $K$, since they share the same task key.

\textbf{Recovery degrades gracefully.} We measure recovery directly under each budget (Table~\ref{tab:pruning}). Retained tasks recover quickly regardless of how aggressively we prune: $3$--$5$ adaptation steps across all budgets and both retention policies, essentially flat as $K$ shrinks from $8$ to $2$. Evicted tasks, by contrast, fall back to cold-start and do not reach the $97\%$ threshold within the $15$-step measurement window, exactly as expected---eviction removes the warm-start, leaving ordinary from-scratch learning. Critically, an evicted task is never \emph{worse} than under a memoryless method, since eviction only removes a warm-start a memoryless baseline never had. We do not protect returning tasks during pruning; retention is decided solely by the stated policy (recency or utility), so these numbers are not favourably biased. The practical takeaway: a fixed cap converts linear growth into a bounded footprint while keeping fast recovery for every retained task; the only cost is that evicted tasks revert to baseline cold-start, which one mitigates by choosing $K$ and the retention policy to match the expected recurrence pattern.

\begin{table}[t]
\centering
\footnotesize
\setlength{\tabcolsep}{3pt}
\caption{Bank pruning on a long cyclic schedule (8 base tasks, 2 cycles, 3 seeds; ResNet-18, LoRA rank 8). Retained-task recovery stays low ($3$--$5$ steps) across all budgets; evicted tasks do not recover within the 15-step window (cold-start). ``$>$15'' = threshold not reached. Bank parameter memory and total memory (incl.\ optional distillation exemplars) both scale linearly and are bounded by the cap $K$.}
\label{tab:pruning}
\begin{tabular}{@{}llcccc@{}}
\toprule
Policy & $K$ & Rec.\,ret. & Rec.\,evict & Bank MB & Total MB \\
\midrule
Recency & $8$ & $4.0 \pm 0.8$ & --- & $2.35$ & $51.7$ \\
Recency & $6$ & $4.4 \pm 1.3$ & $>$15 & $1.76$ & $38.8$ \\
Recency & $4$ & $4.8 \pm 1.4$ & $>$15 & $1.17$ & $25.8$ \\
Recency & $2$ & $2.0 \pm 0.4$ & $>$15 & $0.59$ & $12.9$ \\
\midrule
Utility & $8$ & $3.3 \pm 0.4$ & --- & $2.35$ & $51.7$ \\
Utility & $4$ & $3.3 \pm 0.1$ & $>$15 & $1.17$ & $25.8$ \\
Utility & $2$ & $5.3 \pm 3.3$ & $>$15 & $0.59$ & $12.9$ \\
\bottomrule
\end{tabular}
\end{table}

\subsection{Qualitative: Allocation Evolution}

Figure~\ref{fig:mask} visualises the learned-controller (NSR-RL) allocation mask across the cyclic schedule. Distinct per-task patterns emerge, and returning tasks partially reactivate prior patterns---a qualitative signature of the recall mechanism.

\begin{figure}[t]
    \centering
    \includegraphics[width=\linewidth]{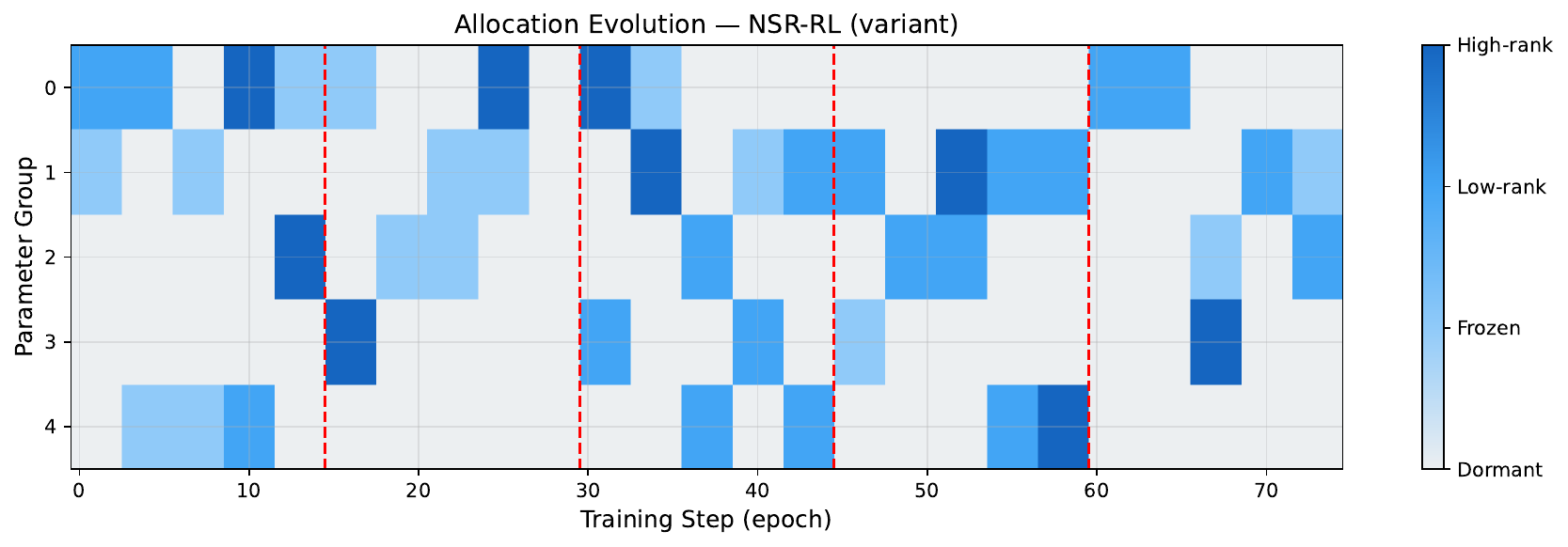}
    \caption{Allocation mask evolution for the learned-controller variant (NSR-RL). The recommended similarity variant instead restores stored masks directly. Red dashed lines mark task boundaries.}
    \label{fig:mask}
\end{figure}

\section{Related Work}
\label{sec:related}

\textbf{Continual learning.} Regularisation (EWC~\citep{kirkpatrick2017overcoming}, SI~\citep{zenke2017continual}), replay (GEM~\citep{lopezpaz2017gradient}), and architecture methods (Progressive Nets~\citep{rusu2016progressive}, PackNet~\citep{mallya2018packnet}) address forgetting differently. NSR is architecture-based but uniquely manages a fixed budget as a \emph{queryable memory of compressed subspaces}, combining distillation and retrieval.

\textbf{LoRA for CL.} O-LoRA~\citep{wang2023olora}, InfLoRA~\citep{liang2024inflora}, PEARL~\citep{bhat2025pearl}, PLAN~\citep{wang2025plan}, and SD-LoRA~\citep{wu2025sdlora} adapt LoRA for continual learning, typically as task-specific, orthogonal, or adaptive-rank adapters. NSR differs by treating each LoRA as a \emph{retrievable, SVD-compressed memory trace}: rather than discarding or merely freezing a task's adapter, it stores the adapter, its mask, head, and embedding, and recalls them by similarity for exact or weighted restoration. This memory view, not a particular allocation rule, is the source of our gains.

\textbf{Retrieval-augmented learning.} HiDe-Prompt~\citep{wang2023hideprompt} and LoRA Hub~\citep{huang2024lorahub} use retrieval/composition, but operate on prompts or pre-trained experts with inference-time optimisation; NSR retrieves \emph{continually-learned} LoRAs via $O(1)$ cosine similarity and restores them directly, with no inference-time search.

\section{Discussion and Limitations}
\label{sec:discussion}

NSR helps most in (1) cyclic environments (Theorem~\ref{thm:greedy-subopt}, Exp 2,3), (2) forgetting-sensitive sequential settings (Exp 1), and (3) heterogeneous cross-domain sequences (Exp 8). Its advantage diminishes in purely sequential settings without recurrence, where there is no stored subspace to recall.

\textbf{Backbone strength and task difficulty.} NSR's benefit is largest when the base representation is limited or the task sequence is hard, so that intelligent reallocation matters. In preliminary experiments with a much stronger pretrained backbone (ViT-B/16), per-task accuracy is already very high ($0.78$--$0.84$ versus $\sim$$0.40$ for ResNet-18 features), and recurring tasks are recovered quickly by \emph{all} policies (warm-start accuracy $\approx 0.80$ regardless of allocation strategy). In this regime simple heuristics become competitive and NSR's margin narrows, because the gap that NSR exploits---the cost of re-learning a forgotten task---is small when the backbone alone nearly solves each task. We note that the Bank's warm-start effect remains observable even here (recovery $4.8$ vs $13.3$ steps with vs without the Bank), but the overall advantage is muted. This delineates NSR's operating regime: it is most valuable under constrained representational capacity or hard, forgetting-prone task sequences, and least necessary when a powerful frozen backbone already generalizes across tasks. A full cross-backbone study is left to future work.

\textbf{What the allocation policy contributes.} Our controlled study (Section~\ref{sec:rl-necessity}) shows that, with the TaskKnowledgeBank in place, a parameter-free similarity rule matches or beats a learned RL controller on accuracy, forgetting, and recovery. We therefore do \emph{not} claim a learned policy is necessary; the recommended NSR uses similarity-based reallocation and trains no RL agent. This is a deliberate, evidence-driven simplification: the theory (Theorem~\ref{thm:greedy-subopt}) requires \emph{history-awareness}, and retrieval from the Bank already supplies it. Casting NSR as \emph{retrieval-based subspace memory management}---rather than ``RL selects LoRA groups''---matches what the experiments actually support and avoids overclaiming. Whether a learned controller helps in settings with only partial task similarity (no exact Bank match) is an open question our current evidence does not resolve in its favour.

\textbf{Memory growth.} The Bank stores one compressed LoRA set per task, so its footprint grows linearly with the number of tasks. Experiment~\ref{sec:pruning} shows this is manageable: the compact Bank \emph{parameter} memory (NSR's primary mechanism) is only $\sim$$0.29$\,MB per task, and a top-$K$ retention policy bounds both parameter and total memory while preserving near-instant recovery for retained tasks; evicted tasks degrade gracefully to cold-start, never worse than a memoryless baseline. If distillation exemplars are used they dominate total memory, so they too are bounded by the same cap. Merging similar tasks (rather than evicting) is a promising refinement. A further limitation is that we evaluate only vision tasks.

\section{Conclusion}
\label{sec:conclusion}

NSR reframes continual learning as retrieval-based memory management over parameter subspaces: it compresses each task's LoRA via SVD, reserves it in a TaskKnowledgeBank, and recalls it by similarity to reallocate the active subspace. Theory and experiments agree: memoryless allocation is structurally suboptimal in cyclic environments, a Bank ablation confirms a $10\times$ recovery effect, and the framework generalises to the heterogeneous 5-Datasets benchmark with the highest accuracy and least forgetting among all baselines. A controlled study shows that the memory mechanism---not a learned controller---is the source of these gains, yielding a simple, reproducible method, and a memory-budget analysis shows the compressed Bank stays small ($0.29$\,MB/task) and boundable by a retention cap. NSR offers a theoretically-motivated, controller-agnostic framework for continual learning under fixed capacity.

\appendix
\section{Full Proof of Theorem~\ref{thm:greedy-subopt}}
\label{app:proof}

\textbf{Step 1.} Decompose $V(\pi)=\sum_{c=1}^M\sum_{j=1}^T V_j^c(\pi)$.

\textbf{Step 2 (memoryless on returning tasks).} For $c\ge2$, by Assumption~\ref{asm:grad} $\pi_{\mathrm{mem}}$ cannot identify $\tau_j$ as returning, so it either selects a suboptimal mask (requiring $\Omega(C_{\mathrm{cold}})$ exploration) or, with probability $p_{\mathrm{lucky}}=O(\binom{N}{k}^{-1})$, the right mask but with overwritten parameters (still requiring cold-start). Thus $\E[\text{steps}] \ge (1-p_{\mathrm{lucky}})C_{\mathrm{cold}}(\varepsilon)$.

\textbf{Step 3 (history-aware).} By Assumption~\ref{asm:hist}, $\bv{h}_t$ identifies $\tau_j$; the Bank restores $(A_i^{\tau_j},B_i^{\tau_j})$; by Assumption~\ref{asm:warm} with $\delta{=}0$, $\E[\text{steps}]\le C_{\mathrm{warm}}(\varepsilon,0)=O(1)$.

\textbf{Step 4 (per-return gap).} $V_j^c(\pi_{\mathrm{hist}})-V_j^c(\pi_{\mathrm{mem}}) \ge \Delta_{\mathrm{switch}} - p_{\mathrm{lucky}}C_{\mathrm{cold}}L_{\mathrm{relearn}}$. For large $N$, $p_{\mathrm{lucky}}\to0$.

\textbf{Step 5 (sum).} Cycle 1 contributes $\ge0$; cycles $2,\ldots,M$ contribute $T(M-1)$ returning instances, yielding $V(\pi_{\mathrm{hist}})-V(\pi_{\mathrm{mem}}) \ge T(M-1)\Delta_{\mathrm{switch}}$. \qed

\section{Hyperparameters}
\label{app:hyper}

\begin{table}[h]
\centering
\small
\begin{tabular}{@{}lc@{}}
\toprule
Parameter & Value \\
\midrule
LoRA max rank & 8 \\
Capacity ratio $k/N$ & 0.5 \\
Compression ratio $\rho$ & 0.5 \\
Distillation $\lambda$ & 0.1 \\
Distillation temperature & 2.0 \\
Base LR & $5\times10^{-3}$ \\
PPO LR & $1\times10^{-3}$ \\
PPO clip & 0.2 \\
Entropy coef & 0.02 \\
Epochs/task & 20 \\
Batch size & 64 \\
Seeds & 42, 123, 456 \\
\bottomrule
\end{tabular}
\end{table}

\bibliographystyle{plainnat}

\begin{thebibliography}{99}
\bibitem[Bhat et~al.(2025)]{bhat2025pearl} Bhat, P. et~al. (2025). PEARL: Parameter-efficient adaptive LoRA for continual learning. \emph{arXiv:2505.11998}.
\bibitem[Bransford et~al.(2000)]{bransford2000how} Bransford, J.~D. et~al. (2000). \emph{How People Learn}. National Academies Press.
\bibitem[French(1999)]{french1999catastrophic} French, R.~M. (1999). Catastrophic forgetting in connectionist networks. \emph{Trends Cogn. Sci.}, 3(4).
\bibitem[He et~al.(2016)]{he2016resnet} He, K. et~al. (2016). Deep residual learning. \emph{CVPR}.
\bibitem[Hu et~al.(2022)]{hu2022lora} Hu, E.~J. et~al. (2022). LoRA. \emph{ICLR}.
\bibitem[Huang et~al.(2024)]{huang2024lorahub} Huang, C. et~al. (2024). LoRAHub. \emph{COLM}.
\bibitem[Kirkpatrick et~al.(2017)]{kirkpatrick2017overcoming} Kirkpatrick, J. et~al. (2017). Overcoming catastrophic forgetting. \emph{PNAS}, 114(13).
\bibitem[Liang et~al.(2024)]{liang2024inflora} Liang, Y. et~al. (2024). InfLoRA. \emph{CVPR}.
\bibitem[Lopez-Paz \& Ranzato(2017)]{lopezpaz2017gradient} Lopez-Paz, D., Ranzato, M. (2017). GEM. \emph{NeurIPS}.
\bibitem[Mallya \& Lazebnik(2018)]{mallya2018packnet} Mallya, A., Lazebnik, S. (2018). PackNet. \emph{CVPR}.
\bibitem[McClelland et~al.(1995)]{mcclelland1995complementary} McClelland, J.~L. et~al. (1995). Complementary learning systems. \emph{Psych. Review}, 102(3).
\bibitem[McCloskey \& Cohen(1989)]{mccloskey1989catastrophic} McCloskey, M., Cohen, N.~J. (1989). Catastrophic interference. \emph{Psych. Learn. Motiv.}, 24.
\bibitem[Olshausen \& Field(2004)]{olshausen2004sparse} Olshausen, B.~A., Field, D.~J. (2004). Sparse coding. \emph{Curr. Opin. Neurobiol.}, 14(4).
\bibitem[Rolnick et~al.(2019)]{rolnick2019experience} Rolnick, D. et~al. (2019). Experience replay for continual learning. \emph{NeurIPS}.
\bibitem[Rusu et~al.(2016)]{rusu2016progressive} Rusu, A.~A. et~al. (2016). Progressive neural networks. \emph{arXiv:1606.04671}.
\bibitem[Schacter et~al.(2007)]{schacter2007constructive} Schacter, D.~L. et~al. (2007). Remembering the past to imagine the future. \emph{Nat. Rev. Neurosci.}, 8(9).
\bibitem[Schulman et~al.(2017)]{schulman2017proximal} Schulman, J. et~al. (2017). PPO. \emph{arXiv:1707.06347}.
\bibitem[Tononi \& Cirelli(2006)]{tononi2006sleep} Tononi, G., Cirelli, C. (2006). Sleep and synaptic homeostasis. \emph{Sleep Med. Rev.}, 10(1).
\bibitem[Wang et~al.(2023a)]{wang2023olora} Wang, Q. et~al. (2023a). O-LoRA. \emph{arXiv:2310.12963}.
\bibitem[Wang et~al.(2023b)]{wang2023hideprompt} Wang, Y. et~al. (2023b). Hierarchical decomposition of prompt-based CL. \emph{NeurIPS}.
\bibitem[Wang et~al.(2025)]{wang2025plan} Wang, X. et~al. (2025). PLAN. \emph{ICCV}.
\bibitem[Wu et~al.(2025)]{wu2025sdlora} Wu, K. et~al. (2025). SD-LoRA. \emph{ICLR}.
\bibitem[Zenke et~al.(2017)]{zenke2017continual} Zenke, F. et~al. (2017). Synaptic intelligence. \emph{ICML}.
\end{thebibliography}

\end{document}